%% file: submission.tex
\documentclass{article}
\usepackage{authblk}

\usepackage{amsmath,amsfonts,amssymb}

\usepackage{amsthm}
\usepackage{times}
\usepackage{bm}
\usepackage{natbib}
\usepackage{xcolor}

\usepackage{pgfplots}
\usepackage{tikz} 
\usepackage{bigints}
\usepackage{graphicx,caption}
\usepackage{float}

\pgfplotsset{compat=newest}
\usetikzlibrary{decorations.pathmorphing} 
\usetikzlibrary{fit} 
\usetikzlibrary{backgrounds} 
\usetikzlibrary{external}
\usetikzlibrary{pgfplots.groupplots}
\usetikzlibrary{positioning}

\usepackage[plain,noend]{algorithm2e}

\newcommand*\de{\mathop{}\!\mathrm{d}}

\DeclareMathOperator*{\EV}{\mathbb{E}}

\newcommand{\svec}[0]{\mathbf{s}}

\newcommand{\approxx}[1]{\hat{#1}}

\newcommand{\yvec}{\mathbf{y}}
\newcommand{\xvec}{\mathbf{x}}

\newcommand{\zvec}{\mathbf{z}}

\newcommand{\lvec}{\mathbf{l}}
\newcommand{\hvec}{\mathbf{h}}

\newcommand{\deltab}{\bm{\delta}}
\newcommand{\gammab}{\bm{\gamma}}
\newcommand{\taub}{\bm{\tau}}

\usepackage{etoolbox}
\providetoggle{detail}
\settoggle{detail}{true}

\newtheorem{theorem}{Theorem}
\newtheorem{definition}{Definition}
\newtheorem{proposition}{Proposition}

\usepackage{enumitem}

\newenvironment{axioms}
{\enumerate[label=\textbf{A\arabic*.}, ref=A\arabic*]}
{\endenumerate}
\makeatletter
\newcommand\varitem[1]{\item[\textbf{A\arabic{enumi}\rlap{$#1$}.}]%
	\edef\@currentlabel{A\arabic{enumi}{$#1$}}}
\makeatother
	
\DeclareMathOperator\erf{erf}
\newcommand{\e}[1]{e^{\textstyle #1}}

\allowdisplaybreaks
\begin{document}

%

\markboth{S. TOSATTO et~al.}{An Upper Bound of the Bias of N.W. Kernel Regression}

\title{An Upper Bound of the Bias of Nadaraya--Watson Kernel Regression under Lipschitz Assumptions}

\author[1]{S. Tosatto}
\author[1]{R. Akrour}
\author[1,2]{J. Peters}
\affil[1]{
	Technische Universit{\"a}t Darmstadt \protect\\
	64289 Darmstadt, Germany \vspace{0.5em}
}
\affil[2]{
	Max Planck Institute for Intelligent Systems \protect\\
	70569 Stuttgart, Germany \protect\\\vspace{0.5em}
	\textsl{\{samuele.tosatto, riad.akrour, jan.peters\}@tu-darmstadt.de}
}
\date{}

\maketitle

\begin{abstract}
The Nadaraya--Watson kernel estimator is among the most popular nonparameteric regression technique thanks to its simplicity. Its asymptotic bias has been studied by Rosenblatt in 1969 and has been reported in a number of related literature. However, Rosenblatt's analysis is only valid for infinitesimal bandwidth.
In contrast, we propose in this paper an upper bound of the bias which holds for finite bandwidths. Moreover, contrarily to the classic analysis we allow for discontinuous first order derivative of the regression function, we extend our bounds for multidimensional domains and we include the knowledge of the bound of the regression function when it exists and if it is known, to obtain a tighter bound.  
We believe that this work has potential applications in those fields where some hard guarantees on the error are needed.
\end{abstract}


\section{Introduction}

Nonparametric regression and density estimation have been used in a wide spectrum of applications, ranging from economics \citep{bansal_nonparametric_1995}, system dynamics identification \citep{wang_gaussian_2006, nguyen-tuong_using_2010}, and reinforcement learning \citep{ormoneit_kernel-based_2002, kroemer_non-parametric_2011,deisenroth_pilco:_2011,kroemer_kernel-based_2012}.
In recent years, nonparameteric density estimation and regression have been dominated by parametric methods such as those based on deep neural networks. These parametric methods have demonstrated an extraordinary capacity in dealing with both high-dimensional data---such as images, sounds or videos---and large dataset. However, it is difficult to obtain strong guarantees on such complex models, which have been shown easy to fool \citep{moosavi-dezfooli_deepfool:_2016}.
Nonparametric techniques have the advantage of being easier to understand, and recent work overcame some of their limitations, by e.g. allowing linear-memory and sub-linear query time for density kernel estimation \citep{charikar_hashing-based-estimators_2017,backurs_space_2019}. These methods allowed nonparameteric kernel density estimation to be performed on datasets of $10^6$ samples and up to $784$ input dimension. As such, nonparametric methods are a relevant choice when one is willing to trade performance for statistical guarantees; and the contribution of this paper is to advance the state-of-the-art on such guarantees.  

Studying the error of a statistical estimator is important.
It can be used for example to tune the hyper-parameters by minimizing the estimated error \citep{hardle_asymptotic_1985,ray_bandwidth_1997,herrmann_choice_1992,kohler_review_2014}. To this end, the estimation error is usually decomposed into an estimation \textsl{bias} and \textsl{variance}. When it is not  possible to derive these quantities, one performs an asymptotic behavior analysis or a convergence to a probabilistic distribution of the error. While all aforementioned analyses give interesting insights on the error and allow for hyper-parameter optimization, they do not provide any strong guarantee on the error, i.e., we are not able to \textsl{upper bound} it with absolute certainty. 

Beyond hyper-parameter optimization, we argue that another important aspect of error analysis is to provide hard (non-probabilistic) bounds of the error for critical data-driven algorithms. We believe that in the close future, learning agents taking autonomous, data-driven, decisions will be increasingly present. These agents will for example be autonomous surgeons, self-driving cars or autonomous manipulators. In many critical applications involving these agents, it is of primary importance to bound the prediction error in order to provide some technical guarantees on the agent's behavior. In this paper we derive hard upper bounds of the estimation error in non-parametric regression with minimal assumptions on the problem such that the bound can be readily applied to a wide range of applications.

Specifically, we consider in this paper the Nadaraya--Watson kernel regression \citep{nadaraya_estimating_1964,watson_smooth_1964}, which can be seen as a conditional kernel density estimate, and we derive an upper bound of the estimation bias  for the Gaussian kernel
under weak local Lipschitz assumptions.
The reason for our choice of estimator falls of its inherent simplicity, in comparison to more sophisticated techniques.
The bias of the Nadaraya--Watson kernel regression has been previously studied by \cite{rosenblatt_conditional_1969}, and has been reported in a number of related work \citep{mack_convolution_1988,fan_design-adaptive_1992,fan_variable_1992,wasserman_all_2006}.
The main assumptions of Rosenblatt's analysis are $h_n \to 0$ (where $n$ is the number of samples) and $n h_n \to \infty$ where $h_n$ is the kernel's  bandwidth. The Rosenblatt's analysis suffers from an asymptotic error $o(h_n^2)$, which means that for large bandwidths it is not accurate. In contrast, we derive an upper bound of the bias of the Nadaraya--Watson kernel regression which is valid for any choice of bandwidth. 

Our analysis is built on weak Lipschitz assumptions \citep{miculescu_sufficient_2000}, which are milder than the (global) Lipschitz, as we require only $|f(x) - f(y)| \leq |x-y|$ $ \forall y \in \mathcal{C}$ given a fixed $x$, instead of the classic $|f(x) - f(y)| \leq |x-y|$ $ \forall y, x \in \mathcal{C}$---where $\mathcal{C}$ is the data domain.
Moreover, the classical analysis requires the knowledge of $m''$, and therefore the continuity of $m'$---where $m''$ and $m'$ are respectively second and first order derivative of the regression function. We relax this assumption, which allows us to obtain a bias upper bound even for functions such as $|x|$, at points where $m''$ is undefined.
When the bandwidth $h_n$ is large, the Rosenblatt's bias analysis, being only valid for $h_n \to 0$, tends to provide wrong estimates of the bias, as we can observe in the experimental section. Furthermore, we consider multidimensional input space, in order to open this analysis to more realistic settings.

\section{Preliminaries}
\input{sections/preliminaries.tex}

\section{Main Result}
\input{sections/technical.tex}

\section{Simulations}
\input{sections/simulation.tex}

\input{sections/conclusion.tex}

\bibliographystyle{biometrika}
\bibliography{zotero}
\newpage
\appendix
\section{Appendix}
\input{sections/appendix.tex}

\end{document}

%% file: sections/preliminaries.tex
Consider the problem of estimating $\EV[Y | X= \xvec]$ where $X \sim f_X$ and $Y = m(X) + \epsilon$, with noise  $\epsilon$, i.e. $\EV[\epsilon] = 0$. The noise can depend on $\xvec$, but since our analysis is conducted point-wise for a given $\xvec$, $\epsilon_\xvec$ will be simply denoted by $\epsilon$. Let $m:\mathbb{R}^d\to \mathbb{R}$ be the \textsl{regression function} and $f_X$ a probability distribution on $X$ called \textsl{design}. In our analysis we consider $X \in \mathbb{R}^d$ and $Y \in \mathbb{R}$. 
The Nadaraya--Watson kernel estimate of $\EV[Y | X=\xvec]$ is 
\begin{align}
	\hat{m}(\xvec) = \frac{\sum_{i=1}^n K_{\hvec}(\xvec - \xvec_i) y_i}{\sum_{j=1}^n K_{\hvec}(\xvec - \xvec_j)},
\end{align} 
where $K_{\hvec}$ is a kernel function with bandwidth-vector $\hvec$, the $\xvec_i$ are drawn from the design $f_X$ and $y_i$ from $m(\xvec_i) + \epsilon$. Note that both the numerator and the denominator are proportional to Parzen-Rosenblatt density kernel estimates \citep{rosenblatt_remarks_1956,parzen_estimation_1962}.
We are interested in the point-wise bias of such estimate $\EV[\hat{m}(\xvec)]- m(\xvec)$. 
In the prior analysis of \cite{rosenblatt_conditional_1969}, knowledge of $m', m'', f_X, f_X'$ is required and $f, m'$ must be continuous in a neighborhood of $x$. In addition, and as discussed in the introduction, the analysis is limited to a one-dimensional design, and for an infinitesimal bandwidth. For clarity of exposition, we briefly present the classical bias analysis of \cite{rosenblatt_conditional_1969} before introducing our results. 
\begin{theorem}{\emph{Classic Bias Estimation \citep{rosenblatt_conditional_1969}.}}
	Let $m\!:\!\mathbb{R}\!\to\!\mathbb{R}$ be twice differentiable.
	Assume a set $\{x_i, y_i\}_{i=1}^n$ of i.i.d. samples from a distribution with non-zero differentiable density $f_X$. Assume $y_i = m(\svec_i) + \sigma(\epsilon_i)$, where $\epsilon_i$ are i.i.d. and zero-mean. The bias of the Nadaraya--Watson kernel in the limit of infinite samples and for $h \to 0$ and $nh_n \to \infty$ is 
	\begin{align}
	\EV \left[\lim_{n\to \infty}\approxx{m}_n(x) \right] - m(x)  & = h_n^2\left(\frac{1}{2}m''(x) + \frac{m'(x)f_X'(x)}{f_X(x)} \right)\int u^2 K(u)\de u + o_P\left(h_n^2\right) \nonumber \\
	& \approx h_n^2\left(\frac{1}{2}m''(x) + \frac{m'(x)f_X'(x)}{f_X(x)} \right)\int u^2 K(u)\de u . \nonumber
	\end{align}
\end{theorem}
The $o_P$ term denotes the asymptotic behavior w.r.t. the bandwidth. Therefore, for a larger value of the bandwidth, the bias estimation becomes worse, as is illustrated in Figure~\ref{figure:simulations}.

%% file: sections/technical.tex
\begin{table}
	\begin{center}
		\begin{tabular}{||c c c c c||} 
			\hline
			Distribution & Density & $\Upsilon$ & $\mathcal{D}$ & $L_f$ \\ [0.5ex] 
			\hline\hline
			$\text{Laplace}(\mu, \lambda)$ & $\frac{1}{2\lambda}\exp{\left(-\frac{|x-\mu|}{\lambda}\right)}$& $(-\infty, +\infty)$ & $(-\infty, +\infty)$ & $\lambda^{-1}$  \\ 
			\hline
			$\text{Cauchy}(\mu; \gamma)$ & $\left(\pi\gamma + \pi\frac{(x-\mu)^2}{\gamma}\right)^{-1}$ & $(-\infty, +\infty)$ & $(-\infty, +\infty)$ & $\frac{2(z-\mu)}{\gamma^2+ (z-\mu)^2}$  \\
			\hline
			$\text{Uniform}(a, b)$ & $\begin{cases}
			\frac{1}{b-a} \quad \text{if} \quad a \leq x \leq b \\
			0 \quad \quad \text{otherwise}
			\end{cases}$ & $(a,b)$ & $(a,b)$ & 0 \\
			\hline
			$\text{Pareto}(\alpha)$ & $\begin{cases}
			\frac{\alpha}{x^{\alpha+1}} \quad \text{if}  x \geq 1 \\
			0 \quad \quad \text{otherwise}
			\end{cases}$ & $(1, +\infty)$ & $(1, +\infty)$ & $1+\alpha$ \\
			\hline
		    $\text{Normal}(\mu, \sigma)$ & $\frac{1}{\sqrt{2\pi\sigma^2}}\exp{-\frac{(x-\mu)^2}{2\sigma^2}}$ & $(-\infty, +\infty)$ & $(a, b)$ & $f_{\mu, \sigma}(a, b)$ \\
			\hline
		\end{tabular}
		
	\end{center}
	\caption{\label{table:distributions} \textsl{Examples of parameters to use for different univariate random design. Note that} $z = \frac{2\mu + \gamma^2 + \sqrt{4\mu + \gamma^4 + 2\mu\gamma^2 - 4(\gamma^2 + \mu + \gamma^2 \mu)}}{2}$ and in the case of the normal function one can find $f_{\mu, \sigma}(a,b) = \max_{y \in \{a, b\}} |\mathcal{N}(y | \mu, \sigma)|$.}
\end{table}
In this section we present two bounds on the bias of the Nadaraya--Watson estimator. The first one considers a bounded regression function $m$, and allows for local Lipschitz conditions on a subset of the design's support. The second bound instead does not require the regression function to be bounded but only the local Lipschitz continuity to hold on all of its support. The definition of ``local'' Lipschitz continuity will be given below. 

In order to develop our bound on the bias for multidimensional inputs, it is important to define some subset of the $\mathbb{R}^d$ space. 
More in detail we consider an open $n$-dimensional interval in $\mathbb{R}^d$ which is defined as $\Omega(\bm{\tau}^-, \bm{\tau}^+) \equiv (\tau_1^-, \tau_1^+) \times \dots  \times (\tau_d^-, \tau_d^+)$ where $\tau^-, \tau^+ \in \overline{\mathbb{R}}^d$.
We now formalize what is meant by weak (log-)Lipschitz continuity. This will prove useful as we need knowledge of the local-Lipschitz constants of $m$ and $\log f_X$ in our analysis. 
\begin{definition}{\emph{Weak Lipschitz continuity at $\xvec$ on the set $\mathcal{C}$ under the $L_1$-norm.\\}}
	Let $\mathcal{C} \subseteq \mathbb{R}^d$ and $f:\mathcal{C}\to \mathbb{R}$. We call $f$ weak Lipschitz continuous at $\xvec \in \mathcal{C}$ if and only if
	\begin{align}
		|f(\xvec) - f(\yvec)| \leq L|\xvec - \yvec|		\quad \forall \yvec \in \mathcal{C}, \nonumber 
	\end{align} 
	where $|\cdot|$ denotes the $L_1$-norm.
\end{definition}
\begin{definition}{\emph{Weak log-Lipschitz continuity at $\xvec$ on the set $\mathcal{C}$ under the $L_1$-norm.\\}}
	Let $\mathcal{C} \subseteq \mathbb{R}^d$. We call $f$ weak log-Lipschitz continuous at $\xvec$ on the set $\mathcal{C}$ if and only if
	\begin{align}
	|\log f(\xvec) - \log f(\yvec)| \leq L|\xvec - \yvec|		\quad \forall \yvec \in \mathcal{C} \nonumber .
	\end{align} 
	Note that the set $\mathcal{C}$ can be a subset of the function's domain.
\end{definition}
It is important to note that, in contrast to the global Lipschitz continuity, which requires 
$|f(\yvec) - f(\zvec)| \leq L |\yvec - \zvec|$ $\forall\yvec,\zvec \in \mathcal{C}$, the weak Lipschitz continuity is defined at a specific point $\xvec$ and therefore allows the function to be discontinuous elsewhere. 
In the following we list the set of assumptions that we use in our theorems.

\begin{axioms}
	\item \label{ax:domain} $f_X$ and $m$ are defined on $\Upsilon \equiv \Omega(\xvec - \upsilon^-, \xvec + \upsilon^+)$ and  $\upsilon^-, \upsilon^+ \in \overline{\mathbb{R}}_+^d$
	\item \label{ax:log-lipschitz} $f_X$ is log weak Lipschitz with constant $L_f$ at $\xvec$ on the set $\mathcal{D} \equiv \Omega(\xvec -\deltab^-,\xvec + \deltab^- ) \subseteq \Upsilon$ and $f_X(\xvec) \geq f_X(\zvec)$ $\forall \zvec \in \Upsilon\backslash\mathcal{D}$ with positive defined $\deltab^-, \deltab^+ \in \overline{\mathbb{R}}_+^d$ (note that this implies $f_X(\yvec) > 0$ $\forall \yvec \in \mathcal{D}$), 
	\item \label{ax:lipschitz} $m$ is weak Lipschitz with constant $L_m$ at $\xvec$ on a the set $\mathcal{G} \equiv \Omega(\xvec -\gammab^-, \xvec + \gammab^+) \subseteq \mathcal{D}$ with positive defined $\gammab^-, \gammab^+ \in \overline{\mathbb{R}}_+^d$,
\end{axioms}

In the following, we propose two different bounds of the bias. The first version considers a bounded regression function ($M < +\infty$), this allows both the regression function and the design to be weak Lipschitz on a subset of their domain. 
In the second version instead, we consider the case of unbounded regression function ($M = +\infty$) or when the bound $M$ is not known. In this case both the regression function and the design must be weak Lipschitz on the entire domain $\Upsilon$. 
\begin{theorem}{\emph{Bound on the Bias with Bounded Regression Function.}\\}
	\label{theo:bounded}
	Assuming \ref{ax:domain}--\ref{ax:lipschitz}, $\hvec \in \mathbb{R}^{d}_+$ a positive defined vector of bandwidths $\hvec = [h_1, h_2, \dots, h_n]^\intercal$, $K_h$ a multivariate Gaussian kernel defined on $\hvec$, $\hat{f}_n(\xvec)$ the Nadaraya--Watson kernel estimate using $n$ observations $\{\xvec_i, y_i\}_{i=1}^n$  with $x_i \sim f_X$, $y_i = m(\xvec_i) + \epsilon_i$ and with noise $\epsilon_i \sim \varepsilon(\xvec_i)$ centered in zero ($\EV[\varepsilon(\xvec_i)] = 0$), $n\to \infty$, and furthermore assuming there is a constant $0\leq M < +\infty$  such that $|m(\yvec) - x(\zvec)| \leq M$ $\forall \yvec, \zvec \in \Upsilon$, the considered Nadaraya--Watson kernel regression bias results to be bounded by 
	\begin{align}
	& \bigg|\EV\Big[\lim_{n\to \infty}\approxx{f}_n(\xvec) \Big] - m(\xvec)\bigg|\nonumber \\
	& \leq \frac{\sum\limits_{k=1}^d {\xi_A}_k \prod\limits_{i\neq k}^d  \zeta(h_i, -\phi_i^-, \phi_i^+)+ M\left(\prod\limits_{i=1}^d\zeta(-\gamma_i^-, \gamma_i^+)  -  \prod\limits_{i=1}^d\zeta(-\phi_i^-, \phi_i^+)+ \prod\limits_{i=1}^d2{\xi_C}_i \right)}{\prod_{i=1}^d \Psi(L_f, h_i, -\delta_i^-, \delta_i^+)} \nonumber
	\end{align}
	where \begin{align}
	& {\xi_A}_k  = \frac{\sqrt{2}L_m h_k}{\sqrt{\pi}} \left(2 - \phi(\phi_k^+, -L_f, h_k)- \phi(-\phi_k^-, L_f, h_k)\right) - L_mL_fh_k^2\zeta(h_k, -\phi_k^-, \phi_k^+) \nonumber, \\
	& \zeta(h, \tau^-, \tau^+) = \e{\frac{L_f^2 h^2}{2}}\left( 2 \varphi(0, L_f, h) - \varphi(-\tau^+, L_f, h) - \varphi(\tau^-, L_f, h)\right) \nonumber,\\
	& \Psi(L, h, \tau^-, \tau^+) = e^{\frac{L^2h^2}{2}}\left(\varphi(\tau^+, L, h) - \varphi(\tau^-, L, h)\right) \nonumber, \\
	& \varphi(l, L, h) = \lim_{g\to l}\erf\left(\frac{g + h^2L}{h\sqrt{2}} \right),  
	\phi(l, L, h) = \lim_{g\to l} e^{-\frac{g^2}{2h^2} - g L} \nonumber
	\end{align}
	$ \phi_i^-=-M/L_m, \phi_i^+=M/L_m$ and $\erf$ is the error function.
\end{theorem}
In the case where $M$ is unknown or infinite, we propose the following bound.
\begin{theorem}{\emph{Bound on the Bias with Unbounded Regression Function.}\\}
	\label{theo:unbounded}
	Assuming \ref{ax:domain}--\ref{ax:lipschitz}, $\hvec \in \overline{\mathbb{R}}_+^{d}$ a positive defined vector of bandwidths $\hvec = [h_1, h_2, \dots, h_n]^\intercal$, $K_h$ a multivariate Gaussian kernel defined on $\hvec$, $\hat{f}_n(\xvec)$ the Nadaraya--Watson kernel estimate using $n$ observations $\{\xvec_i, y_i\}_{i=1}^n$  with $x_i \sim f_X$, $y_i = m(\xvec_i) + \epsilon_i$ and with noise $\epsilon_i \sim \varepsilon(\xvec_i)$ centered in zero ($\EV[\varepsilon(\xvec_i)] = 0$), $n\to \infty$, and furthermore assuming that $\Upsilon \equiv \mathcal{D} \equiv \mathcal{G}$, the considered Nadaraya--Watson kernel regression bias results to be bounded by
	\begin{align}
	   & \bigg|\EV\Big[\lim_{n\to \infty}\approxx{f}_n(\xvec) \Big] - m(\xvec)\bigg|  \leq \frac{\sum\limits_{k=1}^d {\xi_A}_k\prod\limits_{i\neq k}^d  \zeta(h_i, -\upsilon_i^-, \upsilon_i^+)}{\prod_{i=1}^d \Psi(L_f, h_i, -\upsilon_i^-, \upsilon_i^+)} \nonumber
	   \end{align}
	   where $\xi_{A_k}, \zeta, \Psi, \varphi$ are defined as in Theorem~\ref{theo:bounded} and $\phi_i^-=\upsilon_i^-, \phi_i^+=\upsilon_i^+$.
\end{theorem}

The proof of both theorems is detailed in the Supplementary Material.
Note that the conditions required by our theorems are mild and they allow a wide range of random designs, including and not limited to Gaussian, Cauchy, Pareto, Uniform and Laplace distributions. In general every continuously differentiable density distribution is also weak log-Lipschitz in some closed subset of its domain. For example, the Gaussian distribution does not have a finite Lipschitz constant on its entire domain, but on any closed interval, there is a finite weak Lipschitz constant. Examples of densities that are weak log-Lipschitz are presented in Table~\ref{table:distributions}.

%% file: sections/simulation.tex
\begin{figure}[p]
	\includegraphics[width=\linewidth]{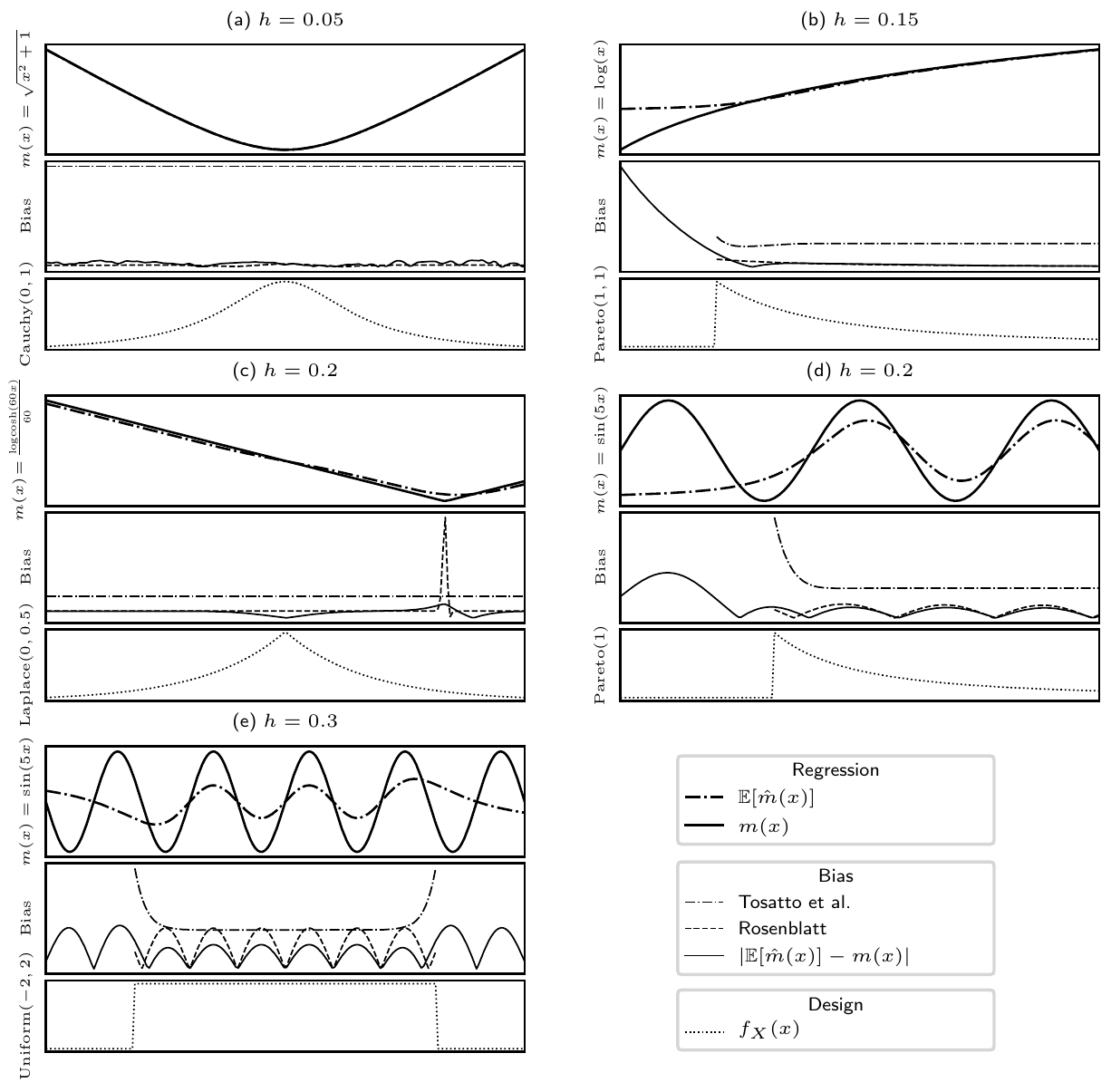}  
	\caption{\label{figure:simulations} 
		\textsl{We propose some simulations of Nadaraya--Watson regression with different designs, regression functions and bandwidths.
			The regression function $m(x)$ is represented with a solid line, while the Nadaraya--Watson estimate $\hat{m}(x)$ is represented with a dash-dotted line in the top subplot of each experiment. 
			In the second subplots, it is possible to observe the true bias (solid line), as well as our upper bound (dashed line) and the Rosenblatt's estimate (dash-dotted line). 
			In the bottom subplots depict the design used.
			The bandwidth used for the estimation is denoted with $h$.
			It is possible to observe that the Rosenblatt's estimate  often under or over estimates the bias. In all the different test conditions, our method correctly upper bounds the bias.}
	}
\end{figure}

In this section we provide a numerical analysis of our bounds on the bias.
We test our method on uni-dimensional input spaces for display purposes. We select a set of regression functions with different Lipschitz constants and different bounds,
\begin{itemize}
	\item $y = \sin(5x)$; $L_m = 5$ and $M=1$,
	\item $y = \log x$ which for $\mathcal{G} \equiv (-1, +\infty)$ has $L_m = 1$ and $M=+\infty$,
	\item $y = 60^{-1} \log\cosh 60 x $ which has $L_m=1$, is unbuounded, and has a particularly high second derivative in $x=0$, with $m''(0) = 60$,
	\item $y = \sqrt{x^2 + 1}$ which has $L_m = 1$ and is unbounded.  
\end{itemize}
In order to provide as many different scenarios as possible we also used the distributions from Table~\ref{table:distributions}, using therefore both infinite domain distributions, such as \textsl{Cauchy} and \textsl{Laplace}, and finite domain such as \textsl{Uniform}. 
In order to numerically estimate the bias, we approximate $E[\hat{m}_n(x)]$ with an ensemble of estimates $N^{-1}\sum_{j=1}^N\hat{m}_{n, j}(x)$ where each estimate $\hat{m}_{n,j}$ is built on a different dataset (drawn from the same distribution $f_X$). 
In order to ``simulate'' $n\to \infty$ we used $n=10^5$ samples. 

In this section we provide some simulations of our bound presented in Theorem~\ref{theo:bounded} and Theorem~\ref{theo:unbounded}, and for the Rosenblatt's case we use
\begin{equation}
	 \left| h_n^2\left(\frac{1}{2}m''(x) + \frac{m'(x)f_X'(x)}{f_X(x)} \right)\int u^2 K(u)\de u \right|\nonumber .
\end{equation}
Since the Rosenblatt's bias estimate is not an upper bound, it can happen that the true bias is higher (as well as lower) than this estimate, as it is possible to see in Figure~\ref{figure:simulations}. 
We presented different scenarios, both with bounded and unbounded functions, infinite and finite design domains, and with larger or smaller choice of bandwidths.
It is possible to observe that, thanks to the knowledge of $f, f', m', m''$ the Rosenblatt's estimation of the bias tends to be more accurate than our bound, however it can happen that it largely overestimate the bias, like in the case of $m(x) = 60^{-1}\log\cosh(60 x)$ in $x=0$ or to underestimate it, most often in boundary regions. 
In contrast, our bound always overestimate the true bias, and despite its lack of knowledge of $f, f', m', m''$, it is most often tight. Moreover, when the bandwidth is small, both our method and Rosenblatt's  deliver an accurate estimation of the bias.
In general, Rosenblatt tends to deliver a better estimate of the bias, but it does not behave as a bound, and in some situations it also can deliver larger mispredictions.
In detail, the plot (a) in Figure~\ref{figure:simulations} shows that, with a tight bandwidth both our method and Rosenblatt's method achieve good approximations of the bias, but only our method correctly upper bounds the bias.
When increasing the bandwidth, we obtain both a larger bias and subsequent larger estimates of the bias. Our method consistently upper bounds the bias, while in many cases Rosenblatt's method under estimates it, especially in proximity of boundaries (subplots b, d, e).
An interesting case can be observed in subplot (c), where we test the function $m(x) = 60^{-1}\log\cosh(60 x)$, which has high second order derivative in $x=0$: in this case, Rosenblatt's method largely overestimates the bias. 
The figure shows that our bound is able to deal with different functions and random designs, being reasonably tight, if compared to the Rosenblatt's estimation which requires the knowledge of the regression function and the design, and respective derivatives.

%% file: sections/conclusion.tex

\section{Acknowledgment}
The research is financially supported by the Bosch-Forschungsstiftung program. 


%% file: sections/appendix.tex
In order to give the proof of the stated Theorems, we deen to introduce some quantities and to state some facts that will be used in our proofs.  
%
%
%
%
%
%

\begin{definition}{Multivariate Gaussian Kernel.\\}
	\label{def:kernel}
	We define the multivariate Gaussian Kernel with bandwidth $\hvec \in \mathbb{R}$ as
	\begin{equation}
	K_{\hvec}(\xvec - \zvec) = \prod_{i=1}^d\frac{\phi(x_i-z_i, 0, h_i)}{\sqrt{2\pi h_i^2}} \nonumber
	\end{equation}
\end{definition}

\begin{definition}{\emph{Integral on a $d$-interval}\\}
	\label{def:orthopodintegration}
	Let $\mathcal{C} \equiv \Omega(\bm{\tau}^-, \bm{\tau}^+)$ with $\bm{\tau}^-, \bm{\tau}^+ \in \overline{\mathbb{R}}^d$. Let the integral of a function $f:\mathcal{C} \to \mathbb{R}$ defined on $\mathcal{C}$ be defined as
	\begin{equation}
	\int_{\mathcal{C}} f(\xvec) \de \xvec = \int_{\tau_1^-}^{\tau_1^+}\int_{\tau_2^-}^{\tau_2^+}\dots \int_{\tau_d^-}^{\tau_d^+}f([x_1, x_2, \dots, x_d]^\intercal) \de x_d \dots \de x_2 \de x_1. \nonumber
	\end{equation}
\end{definition}

\begin{proposition}{}
	\label{prop:loglipschitz}
	There is a function $g:\Upsilon \to \mathbb{R}$ such that
	\begin{equation}
	f_X(\xvec) = \frac{\e{g(\xvec)}}{\int_\Upsilon \e{g(\xvec)}\de \xvec}\nonumber 
	\end{equation}
	and 
	\begin{equation}
	|g(\xvec) - g(\yvec)| \leq L_f |\xvec - \yvec| \quad \forall  \yvec \in \mathcal{D}. \nonumber 
	\end{equation}
\end{proposition}

\begin{proposition}{\emph{Independent Factorization\\}}
	\label{prop:indipendentfactorization}
	Let $\mathcal{C} \equiv \Omega(\taub^-, \taub^+)$ where $\taub^-, \taub^+ \in \mathbb{R}^d$, and $f_i:\mathbb{R}\to \mathbb{R}$, 
	\begin{align}
		\int_{\mathcal{C}} \prod\limits_{i=1}^d f_i(x_i) \de =  \xvec = \prod\limits_{i=1}^d \int_{\mathcal{C}} f_i(x_i) \de \nonumber .
	\end{align}
\end{proposition}

\begin{proposition}
	\label{prop:int1}
	Given $\tau^-, \tau^+ \in \mathbb{R}, h>0$,
	\begin{equation}
		\bigintss_{\tau^-}^{\tau^+} \frac{\e{-\frac{l^2}{2h^2} - l L_f}}{\sqrt{2\pi h^2}}\de l  = 2^{-1}\Psi(L_f, h, \tau^-, \tau^+) \nonumber .
	\end{equation}
\end{proposition}

\begin{proposition}
	\label{prop:int2}
	Given $\tau^- < 0, \tau^+ \geq 0, h> 0$,
	\begin{align}
	\bigintss_{\tau^-}^{\tau^+}  \frac{\e{-\frac{l^2}{2h^2} + |l| L_f}}{\sqrt{2\pi h^2}}\de \lvec &  = 2^{-1} \zeta(h, \tau^-, \tau^+) \nonumber .
	\end{align}
\end{proposition}

\begin{proposition}
	\label{prop:int3}
Given  $\tau^-, \tau^+ \in \mathbb{R}$ and $h > 0$,
\begin{align}
	\int_{\tau^-}^{\tau^+}\frac{\e{-\frac{l^2}{2h^2}- l L_f }}{\sqrt{2\pi h^2}}l L_m \de l & = 
     \frac{L_mh}{\sqrt{2\pi}}\left(\phi(\tau^-, L_f, h) - \phi(\tau^+, L_f, h)\right) - \frac{L_mL_fh^2}{2}\Psi(L_f, h, \tau^-,\tau^+) \nonumber .
\end{align}
\begin{proof}
\begin{align}
    & \int_{\tau^-}^{\tau^+}\frac{\e{-\frac{l^2}{2h^2}- l L_f }}{\sqrt{2\pi h^2}}l L_m \de l =  \frac{L_mh_i}{\sqrt{2\pi}}  \int_{\mathcal{C}}\frac{l}{h^2} \e{-\frac{l^2}{2h^2}- l L_f }\de l \nonumber \\
    & =  \frac{L_m h}{\sqrt{2\pi}}  \left(-\int_{\tau^-}^{\tau^+}\bigg(-\frac{l}{h^2}-L_f\bigg)  \e{-\frac{l^2}{2h^2}- l L_f }\de l   - L_f\sqrt{2\pi h^2}\int_{\tau^-}^{\tau^+}\frac{\e{-\frac{l^2}{2 h^2}- l L_f }}{\sqrt{2\pi h^2}}\de l\right)\nonumber  \\
    & = -\frac{L_m h}{\sqrt{2\pi}} \left[\e{-\frac{l^2}{2h^2}- l L_f }\right]_{\tau^-}^{\tau^+} - \frac{L_mL_fh^2}{2}\Psi(L_f, h, \tau^-, \tau^+). \nonumber 
\end{align}
\end{proof}
\end{proposition}

\begin{proposition}
	\label{prop:int4}
	Given  $\tau^- <0 , \tau^+ \geq 0$ and $h > 0$,
	\begin{align}
	\int_{\tau^-}^{\tau^+}\frac{\e{-\frac{l^2}{2h^2} + |l| L_f }}{\sqrt{2\pi h^2}}|l| L_m \de l & = - \int_{0}^{\tau^+}\frac{\e{-\frac{l^2}{2h^2} - l (-L_f) }}{\sqrt{2\pi h^2}}l (-L_m) \de l - \int_{\tau^-}^{0}\frac{\e{-\frac{l^2}{2h^2} -l L_f }}{\sqrt{2\pi h^2}}l L_m \de l \nonumber \\
	& = \frac{L_m h}{\sqrt{2\pi}} \left(2 - \phi(\tau^+, -L_f, h)- \phi(\tau^-, L_h, h)\right) + \frac{L_mL_fh^2}{2}\zeta(h, \tau^-, \tau^+) \nonumber 
	\end{align}
\end{proposition}

\begin{proposition}
	\label{prop:intprodsum}
Given $\mathcal{C} \equiv \Omega(\bm{\tau}^-, \bm{\tau}^+)$, $p:\mathbb{R} \to \mathbb{R}$, $q:\mathbb{R} \to \mathbb{R}$,
\begin{equation}
    \int_{\mathcal{C}} \bigg(\prod_{i=1}^d p(z_i)\bigg)\bigg(\sum_{k=1}^dg(z_k)\bigg) \de \zvec = \sum_{k=1}^d \bigg(\prod_{i\neq k}^d\int_{\tau_i^-}^{\tau_i^+}p(z)\de z \bigg)\int_{\tau_k^-}^{\tau_k^+}p(z)q(z)\de z. \nonumber 
\end{equation}
\end{proposition}
\begin{proof}{Proof of Theorem~\ref{theo:bounded}}:
\begin{eqnarray}
    \bigg|\EV\Big[\lim_{n\to \infty}\approxx{f}_n(\xvec) \Big] - m(\xvec)\bigg| & = & \bigg|\EV\bigg[\lim_{n\to \infty}\frac{\sum_{i=1}^nK_{\hvec}(\xvec - \xvec_i)y_i}{\sum_{j=1}^nK_{\hvec}(\xvec - \xvec_j)} \bigg] - m(\xvec)\bigg| \nonumber \\
    & = & \bigg|\EV\bigg[\lim_{n\to \infty}\frac{n^{-1}\sum_{i=1}^nK_{\hvec}(\xvec - \xvec_i)y_i}{n^{-1}\sum_{j=1}^nK_{\hvec}(\xvec - \xvec_j)} \bigg] - m(\xvec)\bigg| \nonumber \\
    & = & \bigg|\EV\bigg[\frac{\int_{\Upsilon}K_{\hvec}(\xvec - \zvec)\big(m(\zvec) - \epsilon(\zvec)\big)f_X(\zvec)\de \zvec}{\int_{\Upsilon}K_{\hvec}(\xvec - \zvec)f_X(\zvec)\de \zvec} \bigg] - m(\xvec)\bigg| \nonumber \\
    & = & \bigg|\frac{\int_{\Upsilon}K_{\hvec}(\xvec - \zvec)m(\zvec) f_X(\zvec)\de \zvec}{\int_{\Upsilon}K_{\hvec}(\xvec - \zvec)f_X(\zvec)\de \zvec} - m(\xvec)\bigg| \nonumber \quad \quad \text{\eqref{ax:noise}} \\
    & = & \bigg|\frac{\int_{\Upsilon}K_{\hvec}(\xvec - \zvec)\big(m(\zvec)- m(\xvec) \big)f_X(\zvec)\de \zvec}{\int_{\Upsilon}K_{\hvec}(\xvec - \zvec)f_X(\zvec)\de \zvec}\bigg| \nonumber \\
    & = & \frac{\big|\int_{\Upsilon}K_{\hvec}(\xvec - \zvec)\big(m(\zvec)- m(\xvec) \big)f_X(\zvec)\de \zvec\big|}{\big|\int_{\Upsilon}K_{\hvec}(\xvec - \zvec)f_X(\zvec)\de \zvec\big|}. \nonumber 
\end{eqnarray}
We want to obtain an upper bound of the bias. Therefore we want to find an upper bound of the numerator and a lower bound of the denominator.

\textbf{Lower bound of the Denominator:\\}
The denominator is always positive, so the module can be removed,
\begin{eqnarray}
    \int_{\Upsilon}K_{\hvec}(\xvec - \zvec)f_X(\zvec)\de \zvec & = & \int_{\Upsilon}f_X(\zvec)\prod_{i=1}^d\frac{\e{-\frac{(\xvec_i-\zvec_i)^2}{2\hvec_i^2}}}{\sqrt{2\pi\hvec_i^2}}\de \zvec \nonumber \\
    & \geq & \int_{\mathcal{D}}f_X(\zvec)\prod_{i=1}^d\frac{\e{-\frac{(\xvec_i-\zvec_i)^2}{2\hvec_i^2}}}{\sqrt{2\pi\hvec_i^2}}\de \zvec \nonumber \quad  \text{(since $\mathcal{D} \subseteq \Upsilon $ and the integrand is always non-negative)}\\
    & = & \frac{\e{g(\xvec)}}{\int_{\Upsilon}\e{g(\zvec)}\de \zvec}\int_{\mathcal{D}}\e{g(\zvec)-g(\xvec)}\prod_{i=1}^d\frac{\e{-\frac{(\xvec_i-\zvec_i)^2}{2\hvec_i^2}}}{\sqrt{2\pi\hvec_i^2}}\de \zvec\nonumber \quad \quad \text{(Proposition~\ref{prop:loglipschitz})}\\
    & = &  f_X(\xvec)\int_{\overline{\mathcal{D}}}\e{g(\xvec+\lvec)-g(\xvec)}\prod_{i=1}^d\frac{\e{-\frac{l_i^2}{2 h_i^2}}}{\sqrt{2\pi h_i^2}}\de \lvec  \qquad \text{let $\mathbf{l} = \zvec - \xvec$ and $\overline{\mathcal{D}} \equiv \Omega(-\deltab^-, +\deltab^+)$} \nonumber \\
    & \geq & f_X(\xvec)\int_{\overline{\mathcal{D}}}\e{-|\lvec|L_f}\prod_{i=1}^d\frac{\e{-\frac{l_i^2}{2h_i^2}}}{\sqrt{2\pi h_i^2}}\de \lvec \qquad \text{(\ref{ax:log-lipschitz} + Lipschitz Inequality)} \nonumber \\
    & = & f_X(\xvec)\int_{\overline{\mathcal{D}}}\prod_{i=1}^d\frac{\e{-\frac{l_i^2}{2h_i^2}-l_i L_f}}{\sqrt{2\pi h_i^2}}\de \lvec  \nonumber 
\end{eqnarray}
Now considering Proposition~\ref{prop:indipendentfactorization} and Proposition~\ref{prop:int1}, we obtain
\begin{align}
    \int_{-\infty}^{+\infty}K_{\hvec}(\xvec - \zvec)f_X(\zvec)\de \zvec & \geq f_X(\xvec) 2^{-d}\prod_{i=1}^d \Psi(L_f, h_i, -\delta_i^-, \delta_i^+).
    \label{denominator}
\end{align}

\textbf{Upper bound of the Numerator:\\}
\begin{eqnarray}
    & & \bigg|\int_{\Upsilon}K_{\hvec}(\xvec - \zvec)\big(m(\zvec)- m(\xvec) \big)f_X(\zvec)\de \zvec\bigg| \nonumber \\
     & \leq & \int_{\Upsilon}K_{\hvec}(\xvec - \zvec)\left|m(\zvec)- m(\xvec) \right|f_X(\zvec)\de \zvec \nonumber \\ 
    & = & \int_{\mathcal{G}}K_{\hvec}(\xvec - \zvec)\left|m(\zvec)- m(\xvec) \right|f_X(\zvec)\de \zvec +  \int_{\Upsilon\backslash\mathcal{G}}K_{\hvec}(\xvec - \zvec)\left|m(\zvec)- m(\xvec) \right|f_X(\zvec)\de \zvec \nonumber  \\ 
    & \leq & \int_{\mathcal{G}}K_{\hvec}(\xvec - \zvec)\left|m(\zvec)- m(\xvec) \right|f_X(\zvec)\de \zvec +  f_X(\xvec)M\int_{\Upsilon\backslash\mathcal{G}}K_{\hvec}(\xvec - \zvec)\de \zvec \nonumber \quad \quad \text{\eqref{ax:maximum}}  \\ 
    & = & \frac{\e{g(\xvec)}}{\int_{\Upsilon}\e{g(\zvec)}\de \zvec}\int_{\mathcal{G}}\e{g(\zvec)-g(\xvec)}K_{\hvec}(\xvec - \zvec)\left|m(\zvec)- m(\xvec) \right|\de \zvec +  f_X(\xvec)M\int_{\Upsilon\backslash\mathcal{G}}K_{\hvec}(\xvec - \zvec)\de \zvec \nonumber  \\ 
    & \leq & f_X(\xvec)\left(\int_{\mathcal{G}}\e{g(\zvec)-g(\xvec)}K_{\hvec}(\xvec - \zvec)\left|m(\zvec)- m(\xvec) \right|\de \zvec +  M \xi_C \right) \nonumber  \quad \text{where $\xi_C = \int_{\Upsilon\backslash\mathcal{G}}K_{\hvec}(\xvec - \zvec)\de \zvec$}
\end{eqnarray}
where $\xi_C = 1 - 2^{-d}\prod_{i=1}^d \Psi(0, h_i, -\gamma_i^-, \gamma_i^+ )$ since $\int_{\Upsilon} = 1$.
Let $\mathcal{F} \equiv \Omega(\xvec - \bm{\phi}^-, \xvec + \bm{\phi}^+) \subseteq \mathcal{G}$, we will later define at our convenience.
\begin{align}
& f_X(\xvec)\left(\int_{\mathcal{G}}\e{g(\zvec)-g(\xvec)}K_{\hvec}(\xvec - \zvec)\left|m(\zvec)- m(\xvec) \right|\de \zvec +  M \xi_C \right) \nonumber  \\
=& f_X(\xvec)\left(\int_{\mathcal{F}}\e{g(\zvec)-g(\xvec)}K_{\hvec}(\xvec - \zvec)\left|m(\zvec)- m(\xvec) \right|\de \zvec + \int_{\mathcal{G}\backslash\mathcal{F}}\e{g(\zvec)-g(\xvec)}K_{\hvec}(\xvec - \zvec)\left|m(\zvec)- m(\xvec) \right|\de \zvec +  M \xi_C \right) \nonumber   \nonumber \\
\leq & f_X(\xvec)\left(\int_{\mathcal{F}}\e{g(\zvec)-g(\xvec)}K_{\hvec}(\xvec - \zvec)\left|m(\zvec)- m(\xvec) \right|\de \zvec + M \int_{\mathcal{G}\backslash\mathcal{F}}\e{g(\zvec)-g(\xvec)}K_{\hvec}(\xvec - \zvec)\de \zvec +  M \xi_C \right) \nonumber \\
= & f_X(\xvec)\left(\int_{\overline{\mathcal{F}}}\e{g(\xvec + \lvec)-g(\xvec)}K_{\hvec}(-\lvec)\left|m(\xvec + \lvec)- m(\lvec) \right|\de \lvec + M \int_{\overline{\mathcal{G}}\backslash\overline{\mathcal{F}}}\e{g(\xvec + \lvec)-g(\xvec)}K_{\hvec}(-\lvec)\de \lvec +  M \xi_C \right) \nonumber   \\
& \quad \quad \quad\text{with $\lvec = \zvec - \xvec$, $\overline{F} \equiv \Omega(-\bm{\phi}^-, \bm{\phi}^+)$ and $\overline{G} \equiv \Omega(-\bm{\gamma}^-, \bm{\gamma}^+)$} \nonumber \\
\leq & f_X(\xvec)\left(\int_{\overline{\mathcal{F}}}\e{L_f|\lvec|}K_{\hvec}(\lvec)L_m\left|\lvec\right|\de \lvec + M \int_{\overline{\mathcal{G}}\backslash\overline{\mathcal{F}}}\e{L_f|\lvec|}K_{\hvec}(\lvec)\de \lvec +  M \xi_C \right) \nonumber   \qquad \text{(\ref{ax:log-lipschitz}, \ref{ax:lipschitz} + Lipschitz Inequality)}
\end{align}
The first integral instead can be solved with Proposition~\ref{prop:int2}, Proposition~\ref{prop:int4} and Proposition~\ref{prop:intprodsum},
\begin{align}
& \int_{\overline{\mathcal{F}}}\e{L_f|\lvec|}K_{\hvec}(\lvec)\left|\lvec\right|\de \zvec \\
= &\int_{\overline{\mathcal{F}}} \left(\prod_{i=1}^d \frac{\e{-\frac{l_i^2}{2h_i^2} + |l_i| L_f }}{\sqrt{2\pi h_i^2}}\right)L_m\sum_{i=1}^d |l_i| \de l \de \zvec \nonumber \\
= &\sum_{k=1}^d \left(\prod_{i\neq k}^d\int_{-\phi_i^-}^{\phi_i^+}\frac{\e{-\frac{l_i^2}{2h_i^2} + |l_i| L_f }}{\sqrt{2\pi h_i^2}}\de z \right)\int_{-\phi_k^-}^{\phi_k^+}\frac{\e{-\frac{l_i^2}{2h_i^2} + |l_i| L_f }}{\sqrt{2\pi h_i^2}} L_m |l_i|\de z \nonumber \quad \text{(Proposition~\ref{prop:intprodsum})} \\
= & 2^{-d}\sum_{k=1}^d \left(\prod_{i\neq k}^d  \zeta(h_i, -\phi_i^-, \phi_i^+)\right)\Bigg(\frac{\sqrt{2}L_m h_k}{\sqrt{\pi}} \left(2 - \phi(\phi_k^+, -L_f, h_i)- \phi(-\phi_k^-, L_h, h_i)\right)\nonumber \\
& \quad \quad - L_mL_fh^2\zeta(h_k, -\phi_k^-, \phi_k^+)\Bigg) \nonumber .
\end{align}
The second integral can be solved using Proposition~\ref{prop:int4},
\begin{align}
	\int_{\mathcal{G}\backslash\mathcal{F}}\e{L_f|\lvec|}K_{\hvec}(\lvec)\de \zvec = & \int_{\mathcal{G}}\e{L_f|\lvec|}K_{\hvec}(\lvec)\de\lvec -\int_{\mathcal{F}}\e{L_f|\lvec|}K_{\hvec}(\lvec)\de \lvec \nonumber \\
	=&   2^{-d}\left(\prod_{i=1}^d\zeta(-\gamma_i^-, \gamma_i^+)  -  \prod_{i=1}^d\zeta(-\phi_i^-, \phi_i^+) \right) \nonumber . \\
\end{align}
A good choice for $\mathcal{F}$ is $\phi_i^-=\min(\gamma_i^-, M/L_f)$ and $\phi_i^+=\min(\gamma_i^+, M/L_f)$, as in this way we obtain a tighter bound.  
In last analysis, letting
\begin{align}
	{\xi_A}_k & = \frac{\sqrt{2}L_m h_k}{\sqrt{\pi}} \left(2 - \phi(\phi_k^+, -L_f, h_i)- \phi(-\phi_k^-, L_h, h_i)\right) - L_mL_fh^2\zeta(h_k, -\phi_k^-, \phi_k^+) \nonumber 
\end{align}
we arrive to
\begin{align}
	& \bigg|\EV\Big[\lim_{n\to \infty}\approxx{f}_n(\xvec) \Big] - m(\xvec)\bigg|\nonumber \\
	& \quad \quad \leq \frac{\sum\limits_{k=1}^d {\xi_A}_k\prod\limits_{i\neq k}^d  \zeta(h_i, -\phi_i^-, \phi_i^+) + M\left(\prod\limits_{i=1}^d\zeta(-\gamma_i^-, \gamma_i^+)  -  \prod\limits_{i=1}^d\zeta(-\phi_i^-, \phi_i^+) + 2^d{\xi_C}\right) }{\prod_{i=1}^d \Psi(L_f, h_i, -\delta_i^-, \delta_i^+)} \nonumber 
\end{align}
showing the correctness of Theorem~\ref{theo:bounded}.
\end{proof}
In order to prove Theorem~\ref{theo:unbounded} we shall note that $\Upsilon \equiv \mathcal{G} \equiv \mathcal{F}$ , therefore the lower bound can be  bounded by
\begin{eqnarray}
\int_{\Upsilon}K_{\hvec}(\xvec - \zvec)f_X(\zvec)\de \zvec & = & \int_{\Upsilon}f_X(\zvec)\prod_{i=1}^d\frac{\e{-\frac{(\xvec_i-\zvec_i)^2}{2\hvec_i^2}}}{\sqrt{2\pi\hvec_i^2}}\de \zvec \nonumber \\
& \geq & f_X(\xvec)\int_{\overline{\Upsilon}}\prod_{i=1}^d\frac{\e{-\frac{l_i^2}{2h_i^2}-l_i L_f}}{\sqrt{2\pi h_i^2}}\de \lvec  \nonumber \\
& = & f_X(\xvec) 2^{-d}\prod_{i=1}^d \Psi(L_f, h_i, -\upsilon_i^-, \upsilon_i^+)
\end{eqnarray}

for the numerator, instead
\begin{eqnarray}
& & \bigg|\int_{\Upsilon}K_{\hvec}(\xvec - \zvec)\big(m(\zvec)- m(\xvec) \big)f_X(\zvec)\de \zvec\bigg| \nonumber \\
& \leq & f_X(\xvec)\int_{\Upsilon}\e{g(\zvec)-g(\xvec)}K_{\hvec}(\xvec - \zvec)\left|m(\zvec)- m(\xvec) \right|\de \zvec  \nonumber  \\
&\leq & f_X(\xvec)\int_{\overline{\Upsilon}}\e{L_f|\lvec|}K_{\hvec}(\lvec)L_m\left|\lvec\right|\de \lvec  \nonumber \qquad \text{where $\overline{\Upsilon} \equiv \Omega(\upsilon^-, \upsilon^+)$}
\end{eqnarray}
and therefore, for the reasoning already made for Theorem~\ref{theo:bounded},
\begin{align}
& \bigg|\EV\Big[\lim_{n\to \infty}\approxx{f}_n(\xvec) \Big] - m(\xvec)\bigg|  \leq \frac{\sum\limits_{k=1}^d {\xi_A}_k\prod\limits_{i\neq k}^d  \zeta(h_i, -\upsilon_i^-, \upsilon_i^+)}{\prod_{i=1}^d \Psi(L_f, h_i, -\upsilon_i^-, \upsilon_i^+)} \nonumber
\end{align}
where $\xi_{A_k}, \zeta, \Psi, \varphi$ are defined as in Theorem~\ref{theo:bounded} and $\phi_i^-=\upsilon_i^-, \phi_i^+=\upsilon_i^+$.